\documentclass[11pt,a4paper]{article}
\usepackage[hyperref]{emnlp2020}
\usepackage{times}
\usepackage{latexsym}

\usepackage{microtype}
\usepackage{amsthm,amsmath,amssymb,amsfonts,bbm}
\usepackage{url}
\usepackage[ruled]{algorithm2e}
\usepackage{graphicx}
\usepackage{xspace}
\usepackage{xcolor,colortbl}

\newcommand{\parfrac}[2]{\paran{\frac{#1}{#2}}}
\newcommand{\paran}[1]{\left( #1 \right)}

\usepackage{balance}

\newtheorem{definition}{Definition}

\newtheorem{theorem}{Theorem}

\aclfinalcopy
\title{On Primes, Log-Loss Scores and (No) Privacy}

\author{Abhinav Aggarwal \\
  Amazon Alexa\\
  Seattle, WA USA \\
  \texttt{\small aggabhin@amazon.com} \\\And
  Zekun Xu \\
  Amazon Alexa \\
  Seattle, WA USA \\
  \texttt{\small zeku@amazon.com} \\\And
  Oluwaseyi Feyisetan \\
  Amazon Alexa \\
  Seattle, WA USA\\
  \texttt{\small sey@amazon.com} \\\And
  Nathanael Teissier \\
  Amazon Alexa \\
  Arlington, VA USA\\
  \texttt{\small natteis@amazon.com}}

\date{}

\begin{document}
\maketitle
\begin{abstract}
Membership Inference Attacks exploit the vulnerabilities of exposing models trained on customer data to queries by an adversary. In a recently proposed implementation of an auditing tool for measuring privacy leakage from sensitive datasets, more refined aggregates like the Log-Loss scores are exposed for simulating inference attacks as well as to assess the total privacy leakage based on the adversary's predictions. In this paper, we prove that this additional information enables the adversary to infer the membership of any number of datapoints with full accuracy in a single query, causing complete membership privacy breach. Our approach obviates any attack model training or access to side knowledge with the adversary. Moreover, our algorithms are agnostic to the model under attack and hence, enable perfect membership inference even for models that do not memorize or overfit. In particular, our observations provide insight into the extent of information leakage from statistical aggregates and how they can be exploited. 
\end{abstract}
\section{Introduction}
Protecting customer privacy is of fundamental importance when training ML models on sensitive customer data. While explicit data de-identification and anonymization mechanisms can help protect privacy leakage to some extent, research has shown that this leakage can happen when models trained on customer data can be queried by an external entity~\cite{homer2008resolving,sankararaman2009genomic,li2013membership,shokri2017membership}, or when statistical aggregates on the dataset are exposed~\cite{dwork2010difficulties, dwork2017exposed}. 

Recently, it was shown that the knowledge of Log-Loss scores leaks information about true labels of test datapoints under some constraints on the prior knowledge on these labels~\cite{whitehill2018climbing}. However, extracting meaningful information from these aggregates on arbitrary large datasets, while maintaining reasonable inference accuracy in a limited number of queries to a Log-Loss oracle remained an open problem, specially in cases when no prior knowledge is available. Moreover, the number of queries required by their algorithm scales with the size of the test dataset. We address this problem in this paper and provide multiple algorithms for optimal inference of arbitrarily many test labels in a single query using the exposed Log-Loss scores. This sheds insight into the extent of information leakage from this statistical aggregate and how it can be exploited to game a classification task, for example, in the context of data-mining competitions like Kaggle, KDDCup and ILSVRC Challenge~\cite{russakovsky2015imagenet}.

More concretely, consider the following scenario: you are tasked with a critical binary classification problem. The quality of your solution will be assessed through a performance score (Log-Loss) on an unknown test dataset. If you score the highest among all candidate solutions, then you win a significant cash prize. You are allowed only two attempts at the solution and the best of the two scores will be considered. 

\textit{Is it possible to game this system in a way that your score is always the highest amongst all candidates, without even training any classifier?} 

We answer this in the affirmative by showing that the knowledge of only the size of the test dataset is enough to construct a scheme that can game any binary classifier that uses the Log-Loss metric to assess the quality of classification. This scheme is completely agnostic of the underlying classification task and hence, sheds light on how a malicious modeller can fake a perfect classifier by demonstrating zero test error. We assume that the oracle reports the scores truthfully on the entire dataset.  

A particularly interesting application of our observation is for breaching membership privacy, where an attacker can query the model for inference on a set of datapoints and use these responses to infer what datapoints were used to train that model. Given blackbox access to a model and a data point $x$, this attack model is a binary classifier to infer the membership of $x$ in the training dataset of the target model using its output on $x$ -- the more information this output reveals, the better this inference can be performed. Consequently, the accuracy of the attack depends on how well the adversary can capture the difference in model performance. 

Nonetheless, the popularity of this attack has made it a strong candidate for assessing privacy leakage of models trained on datasets containing sensitive information~\cite{song2019auditing,backes2016membership,pyrgelis2017knock,salem2018ml,liu2019socinf,murakonda2020ml}. A successful attack can compromise the privacy of the users that contribute to the training dataset. Our results show that an oracle access to Log-Loss scores (for example, when using open source privacy auditors on sensitive datasets~\cite{murakonda2020ml}) enables full privacy breach in a single query. 

\subsection{Related Work}
In recent work, \cite{blum2015ladder} demonstrated how an attacker can estimate the test set labels in a competition setting with probability $2/3$. Similarly, and more related to our work, \cite{whitehill2016exploiting,whitehill2018climbing} showed how the knowledge of AUC and Log-Loss scores can be used to make inference on similar test sets by issuing multiple queries for these statistics. Our work extends the latter to optimize the number of queries. Similarly, through a Monte Carlo algorithm,~\cite{matthews2013examination} show how knowing most of the test labels can help estimate the remaining labels upon gaining access to an empirical ROC curve. However, their algorithm is far from exact inference with no \emph{apriori} information of the true labels.

We further observe that the theme of our work is related to two fields of research: adaptive data analysis, and protections of statistical aggregates using Differential Privacy (DP). In adaptive data analysis \cite{hardt2014preventing,dwork2015preserving}, an attacker leverages multiple (adaptive) queries to sequentially construct a complete exploit (e.g., of a test set). Conversely, with DP \cite{dwork2006calibrating} the objective is to protect the aggregate statistics, such as those exploited by \cite{whitehill2016exploiting}, from leaking information. 

\subsection{Log-Loss Metric} 
We begin with reminding the reader of the definition of the Log-Loss metric on a given prediction vector with respect to a binary labeling of the datapoints in the test dataset~\cite{murphy2012machine}.

\begin{definition}[Log-Loss] 
    \label{def:logloss} For a dataset $D = [d_1,\dots,d_{|D|}]$, let $\ell \in \{0,1\}^{|D|}$ be a binary labeling and $\mathbf{x} = [x_1,\dots,x_{|D|}] \in [0,1]^{|D|}$ be a vector of prediction scores. Let $g(\ell_i,x_i) = \ell_i\log_e x_i + (1-\ell_i)\log_e (1-x_i)$. Then, the Log-Loss ($LL$ in short) for $\mathbf{x}$ with respect to $\ell$ is defined as $LL(\mathbf{x},\ell) = -\frac{1}{|D|}\sum_{i=1}^{|D|}g(\ell_i,x_i)$.
\end{definition}
The definition easily generalizes for multi-class classifiers. A common variant is to ignore the normalization by $|D|$. Our constructions in this paper are scale-invariant.

\section{Algorithms for Exact Inference using Log-Loss scores}\label{sec:algorithms}
In this section, we discuss multiple algorithms for single shot inference of all ground-truth labels using carefully constructed prediction vectors that help establish a 1-1 correspondence of the Log-Loss scores with the labelings of the test dataset. We, therefore, refer to the entity that performs such an inference as an \emph{adversary}.

\subsection{Inference using Twin Primes} Our first algorithm uses twin-primes, i.e. pairs of prime numbers within distance $2$ of each other (see OEIS A001359 from \url{https://oeis.org/A001359}.). It has been conjectured that infinitely many such pairs exist~\cite{de1851recherches,dunham2013note}. For a dataset of some finite size $|D| \geq 1$, we require $|D|$ such pairs. The main steps of our approach are outlined in Algorithm~\ref{alg:logloss} and the following theorem proves its correctness.

\begin{algorithm}[t]
Let $5 \leq p_1 < \dots < p_{|D|}$ be a sequence of (smallest) primes such that $p_i+2$ is also a prime for all $i$. Form the prediction vector for $D$ as $\mathbf{x} = \left[ \frac{p_1}{2+p_1},\dots, \frac{p_{|D|}}{2+p_{|D|}} \right]$. Obtain the Log-Loss on $\mathbf{x}$ and use that to infer the ground-truth labels for $D$ using Algorithm~\ref{alg:loglossreconstruction}.
\caption{Inference on dataset $D =
 [d_1,\dots,d_{|D|}]$ using Twin Primes}
\label{alg:logloss}
\end{algorithm}

\begin{theorem}
\label{thm:logloss} If the Twin-Prime Conjecture holds, then for any dataset $D$, the Log-Loss scores returned by Algorithm~\ref{alg:logloss} are in 1-1 correspondence with the binary labelings for datapoints in $D$.
\end{theorem}
\begin{proof}
For a fixed $D$ and labeling $\ell$, it suffices to show that $-|D|\cdot LL(\mathbf{x},\ell)$ takes all unique values. From Definition~\ref{def:logloss}, observe that the following holds: $$-|D|\cdot LL(\mathbf{x},\ell) = \log_e \parfrac{2^{|D_0|}\prod_{d_j \in D_1}p_j}{(2+p_1)\cdots(2+p_{|D|})}.$$

Now, fix any two labelings $\ell_1$ and $\ell_2$ for $D$. If the number of zeros in them are different, then it is easy to see that $\mathbf{x}$ will give different Log-Loss scores on both of them, since the exponent of $2$ in the numerator will be different for these two labelings and all other prime numbers being odd in the denominator, no common factors will exist to cancel this effect. If the number of zeros is the same, then observe that: $$|D| \paran{LL(\mathbf{x},\ell_1)-LL(\mathbf{x},\ell_2)} = \log_e \frac{\prod_{d_i \in D_1^{(2)}}p_i}{\prod_{d_j \in D_1^{(1)}}p_j},$$where $D_1^{(2)}$ is the set of datapoints with label $1$ in $\ell_2$ (similarly for $D_1^{(1)}$). 
Now, since $\ell_1$ and $\ell_2$ are different, there must exist some index $1 \leq k \leq |D|$ for which $\ell_1(k) = 0$ and $\ell_2(k) = 1$. Thus, $p_k$ will appear in the numerator but not in the denominator. Moreover, since the denominator is also a product of primes, it does not divide the numerator in this case, and hence, the difference $LL(\mathbf{x},\ell_1)-LL(\mathbf{x},\ell_2)$ is non-zero.
\end{proof}

As an example of this technique, assume $|D| = 2$. The primes we can use for this construction are $5$ and $11$, so that the prediction vectors can be set as $v_1 = \left[ 5/7, 2/7\right]$ and $v_2 = \left[ 11/13, 2/13\right]$. 
Then, the following lists the log-loss values for the prediction vector $\mathbf{x}^* = [5/7, 11/13]$ (as per Algorithm~\ref{alg:logloss}):
\begin{align*}
    -2LL(\mathbf{x}^*, [0,0]) &= \log_e \frac{2}{7} + \log_e \frac{2}{13} = \log_e \frac{4}{91}\\
    -2LL(\mathbf{x}^*, [0,1]) &= \log_e \frac{2}{7} + \log_e \frac{11}{13} = \log_e \frac{22}{91}\\
    -2LL(\mathbf{x}^*, [1,0]) &= \log_e \frac{5}{7} + \log_e \frac{2}{13} = \log_e \frac{10}{91}\\
    -2LL(\mathbf{x}^*, [1,1]) &= \log_e \frac{5}{7} + \log_e \frac{11}{13} = \log_e \frac{55}{91}
\end{align*}

\begin{algorithm}[t]
\SetAlgoLined
 \textbf{Input: }Log-Loss score $s$\\
 \textbf{Output: }True Labels for datapoints in $D$\\
 Let $e^{s|D|} = p/q$ (lowest form) and $q = 2^m p_1\cdots p_k$. Find the set of (zero-indexed) locations $I$ of primes $p_1\dots p_k$ in OEIS A001359. Construct the labeling $\ell$ as follows: Insert $1$ in indices specified by $I$, and $0$s elsewhere. Return $\ell$.
 \caption{True Labels from Log-Loss}
 \label{alg:loglossreconstruction}
\end{algorithm}

Our construction allows us to give an algorithm to determine the true labeling from the Log-Loss value, without having to consult a lookup table (see Algorithm~\ref{alg:loglossreconstruction}). This follows from Gauss's Fundamental Theorem of Arithmetic (GFoA), that every positive integer is either a prime or is uniquely factorizable as a product of primes~\cite{gauss1966disquisitiones}. We assume that the Log-Loss score $s$ is reported such that $e^{s|D|} = p/q$ is a rational number in its reduced form (i.e. with $q \ne 0$ and $\gcd(p,q)=1$), and that, without loss of generality, the prediction vector was constructed using the first $|D|$ prime numbers, as specified in Algorithm~\ref{alg:logloss}. 

As an example, suppose on a dataset of size $3$, the Log-Loss $s$ is reported such that $e^{3s} = 1729/170$. Note that this requirement of knowing $|D|$ is not necessary, since it is equal to the number of factors of the numerator of $e^{s|D|}$. Now, writing the denominator $170 = 2^1 \times 5 \times 17$, we note that there is $1$ zero in the labeling, and the other two labels are one. From OEIS A001359, we note that $5$ and $17$ are the first and third prime numbers in the series (when we start counting from $5$), respectively, and hence, the first and third datapoints must have labels one. Thus, we have inferred that the true labeling for $D$ must be $[1,0,1]$. 

We acknowledge that the assumption of knowing $e^{s|D|}$ in its reduced fraction form is equivalent to assuming knowledge of $s$ with infinite precision. We defer this investigation to Section~\ref{sec:floating}. 

\subsection{Extension to Multiple Classes} A similar construction can be used to infer all true labels in a multi-class setting as well. For the One-vs-All approach, then it is trivial to see that the individual Log-Loss scores for each class reveal datapoints from that class. For the $K$-ary classifier approach ($K$ being the number of classes), the following construction works: Let $p_1,\dots,p_{|D|}$ be the first $|D|$ primes. For datapoint $d_i$, use the following prediction vector: $$v_i = \left[ 1/\alpha_i,p_i/\alpha_i,\dots,p_i^{K-1}/\alpha_i \right],$$where $\alpha_i = \sum_{j=0}^{K-1}p_i^j$, thus, forming the prediction matrix $v_D = [v_1,\dots,v_{|D|}]$. Given the true labels $\ell \in \{1,\dots,K\}^{|D|}$, it can be shown that the following holds: $$-LL(v_D,\ell) + \sum_{j=1}^K\log_e \alpha_j = \log_e p_1^{\ell_1-1}\cdots p_{|D|}^{\ell_{|D|}-1}.$$This gives the required injection, since the sum on the left is constant for fixed $K$ and $|D|$, and the product on the right is unique (following GFoA).

\subsection{Inference using Binary Representations}
In Algorithm~\ref{alg:logloss}, the main reason why we chose distinct primes was that when the denominator of $e^{s|D|}$ was factorized, the prime factors would uniquely define the locations of 1s in the binary labeling.  The same 1-1 correspondence can be achieved by observing that the each binary labeling is also equivalent to a binary representation (base 2) of a natural number (see Algorithm~\ref{alg:logloss_noTwins}). By using powers of $2$ for only the indices corresponding to locations of $1$s in the binary labeling, when the denominator is now factorized, it produces in the exponent of the $2$ an integer, whose binary representation (when reversed) is exactly the same as the labeling. This also helps eliminate the dependence on the Twin Prime Conjecture. The following theorem formally establishes this proof.

\begin{theorem}
\label{thm:logloss_noTwins}
For any dataset $D$, the Log-Loss scores returned by Algorithm~\ref{alg:logloss_noTwins} are in 1-1 correspondence with the labelings for datapoints in $D$.
\end{theorem}
\begin{proof}

\begin{algorithm}[t]
Form the prediction vector for $D=
 [d_1,\dots,d_{|D|}]$ as $\mathbf{x} = \left[ \frac{\alpha_1}{1+\alpha_1},\dots,\frac{\alpha_{|D|}}{1+\alpha_{|D|}} \right]$, where $\alpha_i = 2^{2^{i-1}}$. Obtain the Log-Loss on $\mathbf{x}$ and use that to infer the true labels for $D$.
\caption{Exact Inference using Binary Representations}
\label{alg:logloss_noTwins}
\end{algorithm}

Similar to the proof of Theorem~\ref{thm:logloss}, for a fixed $D$ and labeling $\ell$, it suffices to show that $d = |D|\cdot LL(\mathbf{x},\ell)$ takes all unique values. Let $I_1$ be the set of indices in $\ell$ that have value $1$. Now, since $x_i = \frac{2^{2^{i-1}}}{1 + 2^{2^{i-1}}}$, we can write the following: $$d = \sum_{j=1}^{|D|}\log_e\paran{1+2^{2^{j-1}}} - \frac{\sum_{i \in I_1} 2^{i-1}}{\log_2 e}.$$ Thus, if $LL(\mathbf{x},\ell_1) = LL(\mathbf{x},\ell_2)$ for two distinct labelings $\ell_1$ and $\ell_2$, then from above, it is easy to see that this can only happen when $\sum_{i \in I_1^{(1)}} 2^{i-1} = \sum_{i \in I_1^{(2)}} 2^{i-1}$, where $I_1^{(j)}$ is the index set (similar to $I_1$) for labeling $\ell_j$. Now, since every positive integer has a unique binary representation, this implies that $I_1^{(1)} = I_1^{(2)}$, which can only happen when the two labelings are the same. Moreover, note that since powers of $2$ are always even, the product in the denominator of the equation above has no common factors with the numerator. Thus, each binary labeling of $\mathbf{x}$ gives a unique Log-Loss score.
\end{proof}

As an example, if the true labels for a dataset $D$ containing four datapoints are $[1,0,1,1]$, the exponent of $2$ can be the natural number represented using the binary representation $1101$, which is 13. Similarly, if the exponent observed is, say 18, then the corresponding binary representation is $10010$, and hence, the true labels must be $[0,1,0,0,1]$. 

\section{Adapting to Fixed Precision Arithmetic}\label{sec:floating}
Can we design prediction vectors such that the Log-Loss scores are atleast some $\Delta$ apart from each other, where $\Delta$ is limited by the floating point precision on the machine used to simulate our inference algorithms?

For distinguishing scores with $\phi$ significant digits, since there are a total of $10^\phi$ possible numeric values, the threshold value of separation is $\Delta \ge 10^{-\phi}$. If the separation in the scores is smaller than this value, then they cannot be distinguished. Inverting this inequality gives $\phi \ge \left \lceil \log_{10} \parfrac{1}{\Delta} \right \rceil.$ For example, if one wishes to have the scores separated by $\Delta \ge 0.2$, then the minimum amount of precision required is $\left \lceil \log_{10} 5 \right \rceil = 1$. For $\Delta = 0.002$, we would need $\phi \ge \left \lceil \log_{10} 500 \right \rceil = 3$ digits.

We can reduce the requirement of a large precision by combining the AUC and Log-Loss scores, which is common in most practical situations where multiple performance metrics are evaluated to give a holistic overview of classifier inference. This way, even if they are individually not-unique but the tuple is unique for each labeling, exact inference can be done. For example, consider a dataset $D = [d_1,d_2,d_3]$ and the prediction vector $v = [0.2, 0.4, 0.6]$. Clearly, neither the AUC scores nor the Log-Loss scores are unique. However, if we consider the two scores together, the labels can be uniquely identified. Moreover, precision of only two significant digits is enough to make this decision.

A rough analysis tells us that with $\phi$ significant digits, there are $10^\phi(10^\phi+1)$ possible unique values in the AUC-Log-Loss tuple (the +1 is to take into account the case when AUC is Not-Defined). Using the pigeonhole principle, for any dataset $D$ with $|D|=n$, a necessary condition for unique inference is that $2^n \leq 10^\phi(10^\phi+1)$, which gives $\phi \gtrsim \left \lceil 0.151n \right \rceil$. Conversely, with a precision of $\phi$ significant digits, one can only hope to uniquely identify labels for datasets of size at most $\left \lfloor \log_2 \paran{10^\phi(10^\phi+1)} \right \rfloor \leq 7\phi$. 

We can recurse over the remaining points in the database in this situation, for exact inference in at most $\lceil n/6\phi \rceil$. For example, using the IEEE 754 double-precision binary floating-point format, which has at least 15 digits precision, at most $\left\lceil \frac{|D|}{90} \right\rceil$ queries suffice.

\section{Exact Membership Inference Attacks using a Log-Loss Oracle} \label{sec:mia}
Our observations provide insight into the extent of information leakage from statistical aggregates and how they can be exploited. A particularly interesting application is designing stronger Membership Inference Attacks. These attacks were first proposed to exploit the vulnerabilities of exposing models trained on customer data to queries by an adversary~\cite{shokri2017membership}. 

In a recently open sourced implementation of an auditing tool for measuring privacy leakage from sensitive datasets~\cite{murakonda2020ml}, more refined aggregates like the Area Under the ROC Curve (AUC) and Log-Loss scores are exposed for simulating inference attacks as well as to assess the total privacy leakage based on the adversary's predictions. In this threat model, our algorithms demonstrate that this additional information enables the adversary to improve its inference accuracy and learn potentially sensitive information about the distribution of data inside sensitive datasets. The response to this query helps infer exactly which datapoints were used for training.

There are multiple observations that one can make about the Log-Loss based attack. First, the adversary never queries the model under attack directly for prediction on any datapoints whatsoever. This makes intuitive sense since the model does not decide what specific data goes into its training. Rather, it is the other way round. Second, the interaction with the model curator is similar to the interaction with the model interface in the attack proposed by~\cite{shokri2017membership} in that the adversary seeks answers to queries that can help leak information about the training data. The only difference is the additional access to a Log-Loss oracle, which helps make our attack purely deterministic.

\section{Conclusion and Future Work}\label{sec:conclusion}
In this paper, we demonstrated how a single Log-Loss query can enable exact inference of ground-truth labels of any number of test datapoints. This sheds light on how sensitive accuracy metrics can be, even when they are computed on arbitrary large datasets and do not intuitively seem to leak any information. 

An interesting question to ask is if other popular metrics (like precision, recall, AUC) used in the ML literature can be exploited for privacy leakage in a similar manner. In~\cite{whitehill2019does}, an AUC-ROC oracle on the test dataset is used to deduce the true labels in at most $2|D|$ queries. This opens up opportunities to explore if exact inference on all datapoints is possible with one AUC query.

Yet another interesting question to ask is if exact inference is possible when the adversary learns only a bound on or an approximate value in each Log-Loss query it issues. Observe that by deciding the size of the test dataset, the adversary also fixes the number of possible values Log-Loss scores can take. In a typical scenario where the adversary has some prior knowledge about the amount by which the reported score differs from the actual value (see~\cite{dwork2019differential} for an approach to add noise to the reported scores), this discrete set of possible scores can present a huge advantage -- the adversary can perform inference over the most likely score under the constraint above. Nonetheless, a distribution over the labelings for the test set can be learnt to bound the inference error.

\balance
\bibliographystyle{acl_natbib}
\bibliography{emnlp2020}

\begin{thebibliography}{26}
\expandafter\ifx\csname natexlab\endcsname\relax\def\natexlab#1{#1}\fi

\bibitem[{Backes et~al.(2016)Backes, Berrang, Humbert, and
  Manoharan}]{backes2016membership}
Michael Backes, Pascal Berrang, Mathias Humbert, and Praveen Manoharan. 2016.
\newblock Membership privacy in microrna-based studies.
\newblock In \emph{Proceedings of the 2016 ACM SIGSAC Conference on Computer
  and Communications Security}, pages 319--330.

\bibitem[{Blum and Hardt(2015)}]{blum2015ladder}
Avrim Blum and Moritz Hardt. 2015.
\newblock The ladder: A reliable leaderboard for machine learning competitions.
\newblock \emph{arXiv preprint arXiv:1502.04585}.

\bibitem[{Dunham(2013)}]{dunham2013note}
William Dunham. 2013.
\newblock A note on the origin of the twin prime conjecture.
\newblock In \emph{Notices of the International Congress of Chinese
  Mathematicians}, volume~1, pages 63--65. International Press of Boston.

\bibitem[{Dwork et~al.(2015)Dwork, Feldman, Hardt, Pitassi, Reingold, and
  Roth}]{dwork2015preserving}
Cynthia Dwork, Vitaly Feldman, Moritz Hardt, Toniann Pitassi, Omer Reingold,
  and Aaron~Leon Roth. 2015.
\newblock Preserving statistical validity in adaptive data analysis.
\newblock In \emph{Proceedings of the forty-seventh annual ACM symposium on
  Theory of computing}, pages 117--126.

\bibitem[{Dwork et~al.(2019)Dwork, Kohli, and Mulligan}]{dwork2019differential}
Cynthia Dwork, Nitin Kohli, and Deirdre Mulligan. 2019.
\newblock Differential privacy in practice: Expose your epsilons!
\newblock \emph{Journal of Privacy and Confidentiality}, 9(2).

\bibitem[{Dwork et~al.(2006)Dwork, McSherry, Nissim, and
  Smith}]{dwork2006calibrating}
Cynthia Dwork, Frank McSherry, Kobbi Nissim, and Adam Smith. 2006.
\newblock Calibrating noise to sensitivity in private data analysis.
\newblock In \emph{Theory of cryptography conference}, pages 265--284.
  Springer.

\bibitem[{Dwork and Naor(2010)}]{dwork2010difficulties}
Cynthia Dwork and Moni Naor. 2010.
\newblock On the difficulties of disclosure prevention in statistical databases
  or the case for differential privacy.
\newblock \emph{Journal of Privacy and Confidentiality}, 2(1).

\bibitem[{Dwork et~al.(2017)Dwork, Smith, Steinke, and
  Ullman}]{dwork2017exposed}
Cynthia Dwork, Adam Smith, Thomas Steinke, and Jonathan Ullman. 2017.
\newblock Exposed! a survey of attacks on private data.

\bibitem[{Gauss(1966)}]{gauss1966disquisitiones}
Carl~Friedrich Gauss. 1966.
\newblock \emph{Disquisitiones arithmeticae}, volume 157.
\newblock Yale University Press.

\bibitem[{Hardt and Ullman(2014)}]{hardt2014preventing}
Moritz Hardt and Jonathan Ullman. 2014.
\newblock Preventing false discovery in interactive data analysis is hard.
\newblock In \emph{2014 IEEE 55th Annual Symposium on Foundations of Computer
  Science}, pages 454--463. IEEE.

\bibitem[{Homer et~al.(2008)Homer, Szelinger, Redman, Duggan, Tembe, Muehling,
  Pearson, Stephan, Nelson, and Craig}]{homer2008resolving}
Nils Homer, Szabolcs Szelinger, Margot Redman, David Duggan, Waibhav Tembe,
  Jill Muehling, John~V Pearson, Dietrich~A Stephan, Stanley~F Nelson, and
  David~W Craig. 2008.
\newblock Resolving individuals contributing trace amounts of dna to highly
  complex mixtures using high-density snp genotyping microarrays.
\newblock \emph{PLoS Genet}, 4(8):e1000167.

\bibitem[{Li et~al.(2013)Li, Qardaji, Su, Wu, and Yang}]{li2013membership}
Ninghui Li, Wahbeh Qardaji, Dong Su, Yi~Wu, and Weining Yang. 2013.
\newblock Membership privacy: a unifying framework for privacy definitions.
\newblock In \emph{Proceedings of the 2013 ACM SIGSAC conference on Computer \&
  communications security}, pages 889--900.

\bibitem[{Liu et~al.(2019)Liu, Wang, Peng, Huang, Li, and
  Cheng}]{liu2019socinf}
Gaoyang Liu, Chen Wang, Kai Peng, Haojun Huang, Yutong Li, and Wenqing Cheng.
  2019.
\newblock Socinf: Membership inference attacks on social media health data with
  machine learning.
\newblock \emph{IEEE Transactions on Computational Social Systems},
  6(5):907--921.

\bibitem[{Matthews and Harel(2013)}]{matthews2013examination}
Gregory~J Matthews and Ofer Harel. 2013.
\newblock An examination of data confidentiality and disclosure issues related
  to publication of empirical roc curves.
\newblock \emph{Academic radiology}, 20(7):889--896.

\bibitem[{Murakonda and Shokri(2020)}]{murakonda2020ml}
Sasi~Kumar Murakonda and Reza Shokri. 2020.
\newblock Ml privacy meter: Aiding regulatory compliance by quantifying the
  privacy risks of machine learning.
\newblock \emph{arXiv preprint arXiv:2007.09339}.

\bibitem[{Murphy(2012)}]{murphy2012machine}
Kevin~P Murphy. 2012.
\newblock \emph{Machine learning: a probabilistic perspective}.
\newblock MIT press.

\bibitem[{de~Polignac(1851)}]{de1851recherches}
Prince~Alphonse de~Polignac. 1851.
\newblock \emph{Recherches nouvelles sur les nombres premiers}.

\bibitem[{Pyrgelis et~al.(2017)Pyrgelis, Troncoso, and
  De~Cristofaro}]{pyrgelis2017knock}
Apostolos Pyrgelis, Carmela Troncoso, and Emiliano De~Cristofaro. 2017.
\newblock Knock knock, who's there? membership inference on aggregate location
  data.
\newblock \emph{arXiv preprint arXiv:1708.06145}.

\bibitem[{Russakovsky et~al.(2015)Russakovsky, Deng, Su, Krause, Satheesh, Ma,
  Huang, Karpathy, Khosla, Bernstein et~al.}]{russakovsky2015imagenet}
Olga Russakovsky, Jia Deng, Hao Su, Jonathan Krause, Sanjeev Satheesh, Sean Ma,
  Zhiheng Huang, Andrej Karpathy, Aditya Khosla, Michael Bernstein, et~al.
  2015.
\newblock Imagenet large scale visual recognition challenge.
\newblock \emph{International journal of computer vision}, 115(3):211--252.

\bibitem[{Salem et~al.(2018)Salem, Zhang, Humbert, Berrang, Fritz, and
  Backes}]{salem2018ml}
Ahmed Salem, Yang Zhang, Mathias Humbert, Pascal Berrang, Mario Fritz, and
  Michael Backes. 2018.
\newblock Ml-leaks: Model and data independent membership inference attacks and
  defenses on machine learning models.
\newblock \emph{arXiv preprint arXiv:1806.01246}.

\bibitem[{Sankararaman et~al.(2009)Sankararaman, Obozinski, Jordan, and
  Halperin}]{sankararaman2009genomic}
Sriram Sankararaman, Guillaume Obozinski, Michael~I Jordan, and Eran Halperin.
  2009.
\newblock Genomic privacy and limits of individual detection in a pool.
\newblock \emph{Nature genetics}, 41(9):965--967.

\bibitem[{Shokri et~al.(2017)Shokri, Stronati, Song, and
  Shmatikov}]{shokri2017membership}
Reza Shokri, Marco Stronati, Congzheng Song, and Vitaly Shmatikov. 2017.
\newblock Membership inference attacks against machine learning models.
\newblock In \emph{2017 IEEE Symposium on Security and Privacy (SP)}, pages
  3--18. IEEE.

\bibitem[{Song and Shmatikov(2019)}]{song2019auditing}
Congzheng Song and Vitaly Shmatikov. 2019.
\newblock Auditing data provenance in text-generation models.
\newblock In \emph{Proceedings of the 25th ACM SIGKDD International Conference
  on Knowledge Discovery \& Data Mining}, pages 196--206.

\bibitem[{Whitehill(2016)}]{whitehill2016exploiting}
Jacob Whitehill. 2016.
\newblock Exploiting an oracle that reports auc scores in machine learning
  contests.
\newblock In \emph{Thirtieth AAAI Conference on Artificial Intelligence}.

\bibitem[{Whitehill(2018)}]{whitehill2018climbing}
Jacob Whitehill. 2018.
\newblock Climbing the kaggle leaderboard by exploiting the log-loss oracle.
\newblock In \emph{Workshops at the Thirty-Second AAAI Conference on Artificial
  Intelligence}.

\bibitem[{Whitehill(2019)}]{whitehill2019does}
Jacob Whitehill. 2019.
\newblock How does knowledge of the auc constrain the set of possible
  ground-truth labelings?
\newblock In \emph{Proceedings of the AAAI Conference on Artificial
  Intelligence}, volume~33, pages 5425--5432.

\end{thebibliography}

\end{document}